\documentclass{article}
\usepackage{graphicx} % more modern
\usepackage{subfig} 
\usepackage{float}
\usepackage{titlesec}

\usepackage[usenames,dvipsnames,dvinames]{xcolor}
\usepackage[bookmarks, colorlinks=true, citecolor=Violet,linkcolor=Mahogany,urlcolor=blue]{hyperref}

\newcommand{\lovasz}{Lov\'asz}
\author{ {\bf Rishabh Iyer} \\
Dept.\ of Electrical Engineering\\  
University of Washington\\ 
Seattle, WA-98175, USA
\and 
{\bf Jeff Bilmes} \\ 
Dept.\ of Electrical Engineering \\  
University of Washington\\ 
Seattle, WA-98175, USA
} 
\usepackage{alltt}  % 

 \usepackage{amsmath,amsthm}
\usepackage[mathscr]{eucal}
\usepackage{amssymb}

\usepackage[protrusion=true,expansion=true]{microtype}
\usepackage{wrapfig}
\usepackage{ifthen}
\newboolean{isdraft}
\setboolean{isdraft}{false} 
\ifthenelse{\boolean{isdraft}}{%
\newcommand{\myaddcomment}[3]{{\color{#1}{\ensuremath{\langle\!\!\langle}{\bf {#2} :} {#3}\ensuremath{\rangle\!\!\rangle}}}}
\newcommand{\rishabh}[1]{\myaddcomment{LimeGreen}{Rishabh}{#1}}
\newcommand{\JTR}[1]{\myaddcomment{LimeGreen}{Jeff\ensuremath{\rightarrow}Rishabh}{#1}}
\newcommand{\jeff}[1]{\myaddcomment{blue}{Jeff}{#1}}
\newcommand{\RTJ}[1]{\myaddcomment{blue}{Rishabh\ensuremath{\rightarrow}Jeff}{#1}}
\newcommand{\toboth}[1]{\myaddcomment{red}{Rishabh \& Jeff}{#1}}
} {
\newcommand{\rishabh}[1]{}
\newcommand{\JTR}[1]{}
\newcommand{\jeff}[1]{}
\newcommand{\RTJ}[1]{}
\newcommand{\toboth}[1]{}
}
\providecommand{\doarxiv}{true}
\usepackage{ifthen}
\newboolean{isarxiv}
\setboolean{isarxiv}{\doarxiv} 
\ifthenelse{\boolean{isarxiv}}{%
\newcommand{\arxiv}[1]{#1}
\newcommand{\notarxiv}[1]{}
}{
\newcommand{\arxiv}[1]{}
\newcommand{\notarxiv}[1]{#1}
}
\newcommand{\arxivalt}[2]{\ifthenelse{\boolean{isarxiv}}{#1}{#2}}
\DeclareMathOperator*{\argmin}{argmin}

\newtheorem{theorem}{Theorem}[section]
\newtheorem{corollary}{Corollary}[theorem]
\newtheorem{lemma}{Lemma}[section]

\newtheorem{example}{Example}[section]

\newtheorem{proposition}{Proposition}[section]

\newcommand{\ri}{\ensuremath{\mbox{reint}}}
\newcommand{\dom}{\ensuremath{\mbox{dom}}}

\newcommand{\cb}{\ensuremath{d_\phi}} %regular continuous bregman
\newcommand{\cbh}{\ensuremath{d_{\phi}^{\mathcal H_{\phi}}}} %generalized continuous bregman
\newcommand{\LB}{\ensuremath{d_{\hat{f}}}} %Lovasz-Bregman divergence
\newcommand{\Sigmas}{\ensuremath{\boldsymbol \Sigma}}
\newcommand{\h}[1]{\ensuremath{\mathcal H_{#1}}} %regular continuous 

\newcounter{propcounter}
\title{The Lov\'asz-Bregman Divergence and connections to \\rank aggregation, clustering, and web ranking\thanks{A shorter version of this appeared in Proc. Uncertainty in Artificial Intelligence (UAI), Bellevue, 2013}}
\begin{document} 
\maketitle
% unless you've provided the [accepted] option to the icml2013

%\icmlkeywords{Bregman divergences, Generalized Bregman divergences, Clustering, Expectation Maximization}

%\vskip 0.3in
%]
% ==========   Preliminary pages
%
%
 
%
% ----- copyright and title pages
%

\begin{abstract}%
%The \lovasz{} Bregman divergences were recently introduced~\cite{rkiyersubmodBregman2012} in the context of clustering ranked vectors. 
  We extend the recently introduced theory of \lovasz{} Bregman (LB)
  divergences~\cite{rkiyersubmodBregman2012} in several ways. We show
  that they represent a distortion between a ``score'' and an
  ``ordering'', thus providing a new view of rank aggregation and order
  based clustering with interesting connections to web ranking. We
  show how the LB divergences have a number of properties akin to many
  permutation based metrics, and in fact have as special cases forms
  very similar to the Kendall-$\tau$ metric. We also show how the LB
  divergences subsume a number of commonly used ranking measures in
  information retrieval, like NDCG~\cite{jarvelin2000ir} and
  AUC~\cite{spackman1989signal}. Unlike the traditional permutation
  based metrics, however, the LB divergence naturally captures a
  notion of ``confidence'' in the orderings, thus providing a new
  representation to applications involving aggregating scores as
  opposed to just orderings. We show how a number of
  recently used web ranking models are forms of \lovasz{} Bregman rank
  aggregation and also observe that a natural form of Mallow's model
  using the LB divergence has been used as conditional ranking models
% has been?
  for the ``Learning to Rank'' problem. \looseness-1
\end{abstract}

\section{Introduction} 
The Bregman divergence first appeared in the context of relaxation
techniques in convex programming~\cite{bregmanoriginal}, and has
found numerous applications as a general framework in clustering~\cite{banerjee2005}, proximal minimization (\cite{censorparallel}),
and others. Many of these applications are
due to the nice properties of the Bregman divergence, and the fact
that they are parameterized by a single convex function.  They also
generalize a large class of divergences between vectors. \arxiv{Recently
Bregman divergences have also been defined between matrices
\cite{tsuda2006matrix}, between functions
\cite{frigyik2008functional} and also between sets~\cite{rkiyersubmodBregman2012}.} \looseness-1

In this paper, we investigate a specific class of Bregman divergences, parameterized via the \lovasz{} extension of a submodular function. Submodular functions are a special class of discrete functions with interesting properties. Let $V$ refer to a finite ground set $\{ 1, 2, \dots, |V|\}$. A set
function $f: 2^V \to \mathbb{R}$ is submodular if $\forall S, T
\subseteq V$, $f(S) + f(T) \geq f(S \cup T) + f(S \cap T)$. Submodular
functions have attractive properties that make their exact or approximate
optimization efficient and often practical~\cite{fujishige2005submodular, rkiyersemiframework2013}. They also naturally arise in
many problems in machine learning, computer vision, economics,
operations research, etc.  A link between convexity and submodularity is seen via
  the \lovasz{} extension (\cite{edmondspolyhedra,lovasz1983}) of the
  submodular function. While submodular functions are growing phenomenon in machine learning, recently there has been an increasing set of applications for the \lovasz{} extension. In particular, recent work~\cite{bach2010structured, frbach1} has shown nice connections between the \lovasz{} extension and structured sparsity inducing norms. \looseness-1

This work is concerned with yet another application of the \lovasz{} extension, in the form of the \lovasz-Bregman divergence. This was first introduced in Iyer \& Bilmes~\cite{rkiyersubmodBregman2012}, in the context of clustering ranked vectors. We extend our work in several ways, mainly theoretically, by both showing a number of connections to the permutation based metrics, to rank aggregation, to rank based clustering and to the ``Learning to Rank'' problem in web ranking.\looseness-1

\subsection{Motivation}
The problems of rank aggregation and rank based clustering are
ubiquitous in machine learning, information retrieval, and social
choice theory. Below is a partial list of some of these
applications.\looseness-1

\paragraph{Meta Web Search: } We are given a collection of search engines, each providing a ranking or a score vector, and the task is to aggregate these to generate a combined result~\cite{lebanon2002cranking}.\looseness-1

\paragraph{Learning to Rank: } The ``Learning to rank'' problem, which is a fundamental problem in machine learning, involves constructing a ranking model from training data. This problem has gained significant interest in web ranking and information retrieval~\cite{liu2009learning}.\looseness-1

\paragraph{Voter or Rank Clustering: }This is an important problem in
social choice theory, where each voter provides a ranking or assigns a
score to every item. A natural problem here is to meaningfully combine
these rankings~\cite{klementiev2008unsupervised}. Sometimes however
the population is heterogeneous and a mixture of distinct populations,
each with its own aggregate representative, fits better.\looseness-1

\paragraph{Combining Classifiers and Boosting: } There has been an
increased interest in combining the output of different systems in an
effort to improve performance of pattern classifiers\arxiv{, something often
used in Machine Translation~\cite{rosti2007combining} and Speech
Recognition\cite{kirchhoff1998combining}}. One way of doing
this~\cite{lebanon2002cranking} is to treat the output of every
classifier as a ranking and combine the individual rankings of weak
classifiers to obtain the overall classification. This is akin to
standard boosting techniques~\cite{freund2003efficient}, except that
we consider rankings rather than just the valuations.

 \subsection{Permutation Based Distance Metrics}
 First a bit on notation -- a permutation $\sigma$ is a bijection from
 $[n] = \{1, 2, \cdots, n\}$ to itself. Given a permutation $\sigma$,
 we denote $\sigma^{-1}$ as the inverse permutation such that
 $\sigma(i)$ is the item assigned rank $i$, while $\sigma^{-1}(j)$ is
 the rank\footnote{This is opposite from the convention used
   in~\cite{lebanon2002cranking, klementiev2008unsupervised,
     kendall1938new, meilua2007consensus} but follows the convention
   of~\cite{fujishige2005submodular}.}  assigned to item
 $j$ and hence $\sigma(\sigma^{-1}(i)) = i$.  We shall use $\sigma_x$
 to denote a permutation induced through the ordering of a vector $x$
 such that $x(\sigma_x(1)) \geq x(\sigma_x(2)) \cdots \geq
 x(\sigma_x(n))$. Without loss of generality, we assume that the
 permutation is defined via a decreasing order of elements.  We shall
 use $v(i), v[i]$ and $v_i$ interchangeably to denote the $i$-th
 element in $v$. Given two permutations $\sigma, \pi$ we can define
 $\sigma \pi$ as the combined permutation, such that $\sigma\pi(i) =
 \sigma(\pi(i))$. Also given a vector $x$ and a permutation $\sigma$,
 define $x\sigma$ such that $x\sigma(i) = x(\sigma(i))$. We also
 define $\sigma x$ as $\sigma x(i) = x(\sigma^{-1}(i))$.

 Recently a number of papers~\cite{lebanon2002cranking,
   klementiev2008unsupervised, kendall1938new, meilua2007consensus}
 have addressed the problem of combining rankings using permutation
 based distance metrics.  Denote $\Sigmas$ as the set of permutations
 over $[n]$. Then $d: \Sigmas \times \Sigmas \rightarrow \mathbb{R}_+$
 is a permutation based distance metric if it satisfies the usual
 notions of a metric, viz.\ $,\forall \sigma, \pi, \tau \in \Sigmas,
 d(\sigma, \pi) \geq 0$ and $d(\sigma, \pi) = 0$ iff $\sigma = \pi$,
 $d(\sigma, \pi) = d(\pi, \sigma)$ and $d(\sigma, \pi) \leq d(\sigma,
 \tau) + d(\tau, \pi)$. In addition, to represent a distance amongst
 permutations, another property which is usually required is that of
 left invariance to reorderings, i.e., $d(\sigma, \pi) = d(\tau \sigma,
 \tau \pi)$\footnote{While in the literature this is called right
   invariance, we have left invariance due to our notation}. The most
 standard notion of a permutation based distance metric is the Kendall
 $\tau$ metric~\cite{kendall1938new}:
\begin{align}
d_T(\sigma, \pi) = \sum_{i, j, i < j} I(\sigma^{-1} \pi(i) > \sigma^{-1} \pi(j))
\end{align}
Where $I(.)$ is the indicator function. This distance metric
represents the number of swap operations required to convert a
permutation $\sigma$ to $\pi$. It's not hard to see that it is a
metric and it satisfies the ordering invariance property.  Other often
used metrics include the Spearman's footrule $d_S$ and rank
correlation $d_R$~\cite{critchlow1985metric}:\looseness-1
\begin{align}
\label{spearman} d_S(\sigma, \pi) \!=\!\! \sum_{i = 1}^n |\sigma^{-1}(i) - \pi^{-1}(i)|
\\
\label{rankc} d_R(\sigma, \pi) \!=\!\! \sum_{i = 1}^n (\sigma^{-1}(i) - \pi^{-1}(i))^2 
\end{align}

 A natural extension to a ranking model is the Mallows model~\cite{mallows1957non}, which is an exponential model defined based on these permutation based distance metrics. This is defined as:
\begin{align}
p(\pi | \theta, \sigma) = \frac{1}{Z(\theta)} \exp(-\theta d(\pi, \sigma)), \;\text{ with }\; \theta \geq 0.
\end{align}
This model has been generalized by~\cite{fligner1986distance} and also extended to multistage ranking by~\cite{fligner1988multistage}. Lebanon and Lafferty~\cite{lebanon2002cranking} were amongst the first to use these models in machine learning by proposing an extended mallows model~\cite{fligner1986distance} to combine rankings in a manner like adaboost~\cite{freund2003efficient}. Similarly Meila \textit{et al}~\cite{meilua2007consensus} use the generalized Mallows model to infer the optimal combined ranking. Another related though different problem is clustering ranked data, investigated by~\cite{murphy2003mixtures}, where they provide a $k$-means style algorithm. This was also extended to a machine learning context by~\cite{busse2007cluster}. \looseness-1

\subsection{Score based Permutation divergences}
In this paper, we motivate another class of divergences, which capture
the notion of distance between permutations. Unlike the permutation
based distance metrics, however, these are distortion functions
between a ``score'' and a permutation.  This, as we shall see, offers
a new view of rank aggregation and order based clustering problems. We
shall also see a number of interesting connections to web
ranking.\looseness-1
  
Consider a scenario where we are given a collection of scores $x^1,
x^2, \cdots, x^n$ as opposed to just a collection of orderings --
i.e., each $x^i$ is an ordered vector and not just a permutation. This
occurs in a number of real world applications. For example, in the
application of combining classifiers~\cite{lebanon2002cranking}, the
classifiers often output scores (in the form of say normalized
confidence or probability distributions). While the rankings
themselves are informative, it is often more beneficial to use the
additional information in the form of scores if available. This in
some sense combines the approach of
Adaboost~\cite{freund2003efficient} and
Cranking~\cite{lebanon2002cranking}, since the former is concerned
only with the scores while the latter takes only the orderings.  The
case of voting is similar, where each voter might assign scores to
every candidate (which can sometimes be easier than assigning an
ordering). This also applies to web-search where often the individual
search engines (or possibly features) provide a confidence score for
each webpage. Since these applications provide both the valuations and
the rankings, we call these \emph{score based ranking
  applications.}\looseness-1

A \emph{score based permutation divergence} is defined as
follows. Given a convex set $\mathbb{S}$, denote $d: \mathbb{S} \times
\Sigmas \rightarrow \mathbb{R}_+$ as a score based permutation
divergence if $\forall x \in \mathbb{S}, \sigma \in \Sigmas, d(x ||
\sigma) \geq 0$ and $d(x || \sigma) = 0$ if and only if $\sigma_x =
\sigma$. Another desirable property is that of left invariance,
viz.\ $d(x || \sigma)= d(\tau x || \tau \sigma), \forall \tau, \sigma
\in \Sigmas, x \in \mathbb{S}$.

It is then immediately clear how the score based permutation
divergence naturally models the above scenario. The problem 
becomes one of finding a representative ordering, i.e., find a
permutation $\sigma$ that minimizes the average distortion to the
set of points $x^1, \cdots, x^n$. Similarly, in a clustering
application, to cluster a set of ordered scores, a score based
permutation divergence fits more naturally. The representatives for
each cluster are permutations, while the objects themselves are
ordered vectors. Notice that in both cases, a purely permutation based
distance metric would completely ignore the values, and just consider
the induced orderings or permutations. To our knowledge, this work is the first time that the notion of a score based permutation divergence has been introduced formally, thus providing a novel view to the rank aggregation and rank based clustering problems. 
 
\subsection{Our Contributions}

In this paper, we investigate several theoretical properties of one
such score based permutation divergence -- the LB
divergence. This work builds on our previous work~\cite{rkiyersubmodBregman2012}, where we introduce the
\lovasz-Bregman divergence. Our focus therein is mainly on the connections
between the \lovasz{} Bregman and a discrete Bregman divergence connected with submodular functions and we also provide a k-means framework for clustering
ordered vectors. In the present paper, however, we make the connection to rank
aggregation and clustering more precise, by motivating the class of
score based permutation divergences and showing relations to
permutation based metrics and web ranking. \arxivalt{We show several new theoretical properties and
interesting connections, summarized below:\looseness-1

\begin{itemize}
\item We introduce a novel notion of the generalized Bregman divergence based on a ``subgradient map''. While this is of independent theoretical interest, it helps us characterize the \lovasz-Bregman divergence.\looseness-1

\item We show that the LB divergence is indeed a score based permutation divergence with several similarities to permutation based metrics. \looseness-1

\item We show that a form of weighted Kendall $\tau$, and a form related to the Spearman's Footrule, occurs as instances of the \lovasz-Bregman divergences. \looseness-1

\item We also show how a number of loss functions used in IR and web ranking like the Normalized Discounted Cumulative Gain (NDCG) and the Area Under the Curve (AUC) occur as instances of the \lovasz{} Bregman.

\item We show some unique properties of the LB divergences not present in permutation-distance metrics. Notable amongst these are the properties that they naturally captures a notion of ``confidence'' of an ordering, and exhibit a priority for higher rankings, both of which are desirable in score based ranking applications.  \looseness-1

\item We define the \lovasz{}-Mallows model as a conditional model over both the scores and the ranking. We also show how the LB divergence can be naturally extended to partial orderings.\looseness-1

\arxiv{\item We show an application of the LB divergence as a convex regularizer in optimization problems sensitive to rankings, and also certain connections to the structured norms defined in~\cite{frbach1}.}

\item We finally connect the LB divergence to rank aggregation and rank based clustering. We show in fact that a number of ranking models for web ranking used in the past are instances of \lovasz{} Bregman rank aggregation. We also show that a number of conditional models used in the past for learning to rank are closely related to the \lovasz{}-Mallows model.
\end{itemize}}{

The following are some of our main results. We introduce a novel
notion of the generalized Bregman divergence based on a ``subgradient
map''. While this is of independent theoretical interest, it helps us
characterize the \lovasz-Bregman divergence. We then show that the LB
divergence is indeed a score based permutation divergence with several
similarities to permutation based metrics. In fact, we show that a
form of weighted Kendall $\tau$, and a form related to the Spearman's
Footrule, occurs as instances of the \lovasz-Bregman divergences.  We
also show how a number of loss functions used in IR and web ranking
like the Normalized Discounted Cumulative Gain
(NDCG)~\cite{jarvelin2000ir} and the Area Under the Curve
(AUC)~\cite{spackman1989signal} occur as instances of the LB. We then
demonstrate some unique properties of the \lovasz{} Bregman divergence
not present in permutation-distance metrics. Notable amongst these are
the properties that the \lovasz-Bregman naturally captures a notion of
``confidence'' of an ordering, and exhibits a priority for higher
rankings, both of which are desirable in score based ranking
applications.  We then define a \lovasz{}-Mallows model as a
conditional model over both the scores and the ranking. We finally
connect the \lovasz{} Bregman to rank aggregation and rank based
clustering. We show in fact that a number of ranking models for web
ranking used in the past are instances of \lovasz{} Bregman rank
aggregation. We moreover show that a number of conditional models used
in the past for learning to Rank are closely related to the
\lovasz{}-Mallows model.\looseness-1}

\section{The \lovasz{} Bregman divergences}
In this section, we shall briefly review the \lovasz{} extension and
define forms of the generalized Bregman and the LB
divergence. \notarxiv{We only state the main results here and for a
  more extensive discussion, refer to~\cite{rkiyerlovBregmanext2013}.}
%It will be clear here that the \lovasz{} Bregman divergence is indeed a partial permutation based divergence, thus providing new insights into problems involving voter clustering and rank aggregation.

\subsection{The Generalized Bregman divergences} \label{GenBregSec}
The notion used in this section follows from~\cite{rockafellar1970convex, telgarsky2012agglomerative}. We denote $\phi$ as a proper convex function (i.e., it's domain is non-empty and it does not take the value $-\infty$), $\ri(.)$ and $\dom(.)$ as the relative interior and domain respectively. \arxiv{Recall that $\dom(x) = \{x \in \mathbb{R}^n: \phi(x) < \infty\}$ and the relative interior of a set $S$ is defined as $\ri(S) = \{x \in S : \exists \epsilon > 0, B_{\epsilon}(x) \cap \text{aff}(S) \subseteq S\}$ where $B_{\epsilon}(x)$ is a ball of radius $\epsilon$ around $x$ and $\text{aff}(S)$ is the affine hull of $S$. } 
A subgradient $g$ at $y \in \dom(\phi)$ is such that for any $x, \phi(x) \geq \phi(y) + \langle g, x - y \rangle$ and the set of all subgradients at $y$ is the subdifferential and is denoted by $\partial_{\phi}(y)$.\looseness-1

The Taylor series approximation of a twice differentiable convex function provides a
natural way of generating a 
Bregman divergence~(\cite{bregmanoriginal}). 
Given a twice differentiable and strictly convex function $\phi$: %the Bregman divergence is defined as:
\begin{align} \label{eq:Bregman}
\cb(x, y) = \phi(x) - \phi(y) - \langle \nabla \phi(y), x - y \rangle.
\end{align}

In order to extend this notion to non-differentiable convex functions, generalized Bregman divergences have been proposed~\cite{telgarsky2012agglomerative, kiwielproximal}. While gradients no longer exist at points of non-differentiability, the directional derivatives exist in the relative interior of the domain of $\phi$, as long as the function is finite. Hence a natural formulation is to replace the gradient by the directional derivative, a notion which has been pursued in~\cite{telgarsky2012agglomerative, kiwielproximal}. 

In this paper, we view the generalized Bregman divergences slightly differently, in a way related to the approach in~\cite{gordon}. In order to ensure that the subgradients exist, we only consider the relative interior of the domain. Then %given $y \in \ri(\dom(\phi))$, such that $\partial_{\phi}(y)$ exists, 
define $\h{\phi}(y)$ as a subgradient-map such that $\forall y \in \ri(\dom(\phi)), \h{\phi}(y) \in \partial_{\phi}(y)$. Then given $x \in \dom(\phi), y \in \ri(\dom(\phi))$ and a subgradient map $\h{\phi}$, we define the generalized Bregman divergence as:\looseness-1
\begin{align}
\label{eq:genBregman}
\cbh(x, y) = \phi(x) - \phi(y) - \langle \h{\phi}(y), x - y \rangle
\end{align}
When $\phi$ is differentiable, notice that $\partial_{\phi}(y) = \{\nabla \phi(y)\}$ and hence $\h{\phi}(y) = \nabla(y)$.  \arxiv{This notion is related to the variant defined in~\cite{gordon}. Given a convex function $\phi$, with $\phi^*$ referring to it's Fenchel conjugate (defined as $\phi^*(y) = \sup_{x} \langle x, y \rangle - \phi(x), \forall y \in \mbox{dom}(\phi^*)$), and a subgradient $g \in \mathbb{R}^n$, a variant of the generalized Bregman defined in~\cite{gordon} is:\looseness-1
\begin{align}\label{gordoneq}
B_{\phi}(x, g) = \phi(x) + \phi^*(g) - \langle g, x \rangle.
\end{align}
It follows then from Fenchel-Young duality (\cite{rockafellar1970convex}, Theorem 23.5) that  $B(x, \h{\phi}(y)) = \cbh(x, y)$. 
We denote the class of divergences defined through Eqn.~\eqref{eq:genBregman} as $\mathcal D^S$. This class admits an interesting characterization:
\begin{theorem}\label{thm2}
$\mathcal D^S$ is exactly the class of divergences $d$ which satisfy the conditions that, (i) for any $a$, $d(x, a)$ is convex in $x$ and (ii) for any given vectors $a, b$, $d(x, a) - d(x, b)$ is linear in $x$.
\end{theorem}
\begin{proof}
It's not hard to check that any divergence in $\mathcal D^S$ satisfies the convex-linear property. We now show that any divergence, satisfying the convex-linear property belongs to $\mathcal D^S$. Given a vector $a$, $d(x, a)$ is convex in $x$ from (i), and hence let $\phi_a(x) = d(x, a)$. Further from (ii), $d(x, b) - d(x, a) = -\langle h, x \rangle + c$, for some vector $h \in \mathbb{R}^n$ and some constant $c \in \mathbb{R}$. Hence $d(x, b) = \phi_a(x) - \langle h, x \rangle + c$. Now substituting $d(b, b) = 0$, we have
\begin{align}
d(b, b) = \phi_a(b) - \langle h, b \rangle + c = 0
\end{align}
and hence $c = -\phi_a(b) + \langle h, b \rangle$. Substituting this back above, we have:
\begin{align}
d(x, b) &= \phi_a(x) - \langle h, x \rangle -\phi_a(b) + \langle h, b \rangle \\
	&= \phi_a(x) - \phi_a(b) - \langle h, x - b\rangle
\end{align}
Since $d(x, b) \geq 0$, we have that $h \in \partial \phi_a(b)$. Hence we can define $\mathcal S_{\phi_a}(b) = h$. Since this holds for every $b \in \mathbb{S}$, we have that the convex-linear divergence belongs to $\mathcal D^S$ for an appropriate sub-gradient map. \looseness-1
\end{proof}
The above result provides the necessary and sufficient conditions for the generalized Bregman divergences. Moreover, this class is strictly more general than the class of Bregman divergences. In order to have a better understanding of the different subgradients at play, we provide a simple example of the generalized Bregman divergences.
\begin{example} \label{ex1}
Let $\phi: \mathbb{R}^n \rightarrow \mathbb{R}: \phi(x) = ||x||_1$. The points of non-differentiability are when $x_i = 0$. Define a subgradient map such that $\h{\phi}(0) = 0$, which is a valid subgradient at $x = 0$. We can obtain the expression in this case following simple geometry (see~\cite{telgarsky2012agglomerative}). With the above subgradient map, we obtain $\cbh(x, 0) = ||x||_1$ and thus get the $l_1$ norm. Other choices of subgradients give different valuations for $d(x, 0)$. For example, the generalized Bregman divergence proposed in~(\cite{telgarsky2012agglomerative}), gives $d(x, 0) = 2||x||_1$. 
\end{example}
}

\subsection{Properties of the \lovasz{} Extension}
We review some important theoretical properties of the \lovasz{} extension. %The \lovasz{} extention of a submodular function is defined as follows. 
Given any vector $y \in [0, 1]^n$ and it's associated permutation $\sigma_y$, %such that
%$y[\sigma_y(1)] \geq y[\sigma_y(2)] \geq \cdots \geq y[\sigma_y(n)]$
 define $S^{\sigma_y}_j = \{\sigma_y(1), \cdots, \sigma_y(j)\}$ for $j \in [n]$. Notice that in general $\sigma_y$ need not be unique (it will be unique only if $y$ is totally ordered), and hence let $\Sigmas_y$ represent the set of all possible permutations with this ordering. Then the \lovasz{} extension of $f$ is defined as:
\begin{align}
\hat{f}(y) = \sum_{j = 1}^n y[\sigma_y(j)]( f(S^{\sigma_y}_j) - f(S^{\sigma_y}_{j-1}))
\end{align}
This is also called the Choquet integral~\cite{choquet1953theory} of $f$. Though $\sigma_y$ might not be unique\arxiv{ (in that there maybe many possible permutations corresponding to the ordering of $y$)}, the \lovasz{} extension is actually unique. Furthermore, $\hat{f}$ is convex if and only if $f$ is submodular.
%, in which case $\hat{f}$ is identical to $\tilde{f}$.
%\begin{lemma}~\cite{lovasz1983, fujishige2005submodular}
%$\hat{f}(y)$ is convex if and only if $f$ is submodular. Furthermore if $f$ is submodular, $\tilde{f}(y) = \hat{f}(y)$.
%\end{lemma}
In addition, the \lovasz{} extension is also tight on the vertices of the hypercube, in that $f(X) = \hat{f}(1_X), \forall X \subseteq V$ (where $1_X$ is the characteristic vector of $X$, i.e.,   $1_X(j) = I(j \in X)$) and hence it is a valid continuous extension. The \lovasz{} extension is in general a non-smooth convex function, and hence there does not exist a unique subgradient at every point. \arxiv{The subdifferential $\partial \hat{f}(y)$ corresponding to the \lovasz{} extension however has a particularly interesting structure. 

It is instructive to consider an alternative representation of the \lovasz{} extension. Let $\emptyset = \mathit Y_0 \subset \mathit Y_1 \subset \mathit Y_2 \subset \cdots \subset \mathit Y_k$ denote a unique chain corresponding to the point $y$, such that $y = \sum_{j = 1}^k \lambda_j 1_{\mathit Y_j}$. Note that in general $ k \leq n$ with equality only if $y$ is totally ordered. Then the \lovasz{} extension can also~ be expressed as\cite{fujishige2005submodular}: $\hat{f}(y) = \sum_{j = 1}^k\lambda_j f(\mathit Y_j)$. \arxiv{Furthermore, we define $\partial_f(Y)$ as the subdifferential of a submodular function $f$, which satisfies $\partial_f(Y) = \partial_{\hat{f}}(1_Y)$.\footnote{Note that this is not the standard definition of $\partial_f(Y)$, but an alternate representation -- see theorem 6.16~\cite{fujishige2005submodular}}

When defined in this form, the subdifferential corresponding to the \lovasz{} extension has a particularly nice form:
\begin{lemma}~(Theorem 6.17, \cite{fujishige2005submodular}): 
For a submodular function $f$ and vector $y \in [0, 1]^n$,
\begin{align}
\partial \hat{f}(y) = \cap\{ \partial_f(\mathit Y_i) | i = 1, 2 \cdots,
k\}.
\end{align}
%If $y = 1_Y$ is a vertex of the hypercube, $\partial \hat{f}(y) = \partial_f(Y)$.
\end{lemma}}

%We now point out that the two definitions of the \lovasz{} extention are equivalent and infact related. 
Let $\Sigmas_{[A, B]}$ represent the set of all possible permutations of the elements in $B \backslash A$. Then we have the following fact, which is easy to verify:
\begin{proposition}
For a submodular function $f$ and a vector $y$, 
\begin{align}
\Sigmas_y = \Sigmas_{[\mathit Y_0, \mathit Y_1]} \times \Sigmas_{[\mathit Y_1, \mathit Y_2]} \times \cdots \Sigmas_{[\mathit Y_{k-1}, \mathit Y_k]}
\end{align}
Moreover, $|\Sigmas_y| = \Pi_{i = 1}^{k} |\mathit Y_i \backslash \mathit Y_{i-1}|$.
\end{proposition} 
Hence we can see now that if the point $y$ is totally ordered, $k = n$.} %and the \lovasz{} subdifferential $\partial_{\hat{f}}$ will just consist of a single point. 
The following result due~\cite{fujishige2005submodular, edmondspolyhedra} provides a characterization of the extreme points of the \lovasz{} subdifferential polyhedron $\partial \hat{f}(y)$:
\begin{lemma}~\cite{fujishige2005submodular, edmondspolyhedra} \label{lem2}
For a submodular function $f$, a vector $y$ and a permutation $\sigma_y \in \Sigmas_y$, a vector $h^f_{\sigma_y}$ defined as:
\begin{align}
h^f_{\sigma_y}(\sigma_y(j)) = f(S^{\sigma_y}_j) - f(S^{\sigma_y}_{j-1}), \forall j \in \{1, 2, \cdots, n\} \nonumber
\end{align}
forms an extreme point of $\partial \hat{f}(y)$. 
Also, the number of extreme points of $\partial \hat{f}(y)$ is $|\Sigmas_y|$.
\end{lemma}
Notice that the extreme subgradients are parameterized by the permutation $\sigma_y$ and hence we refer to them as $h^f_{\sigma_y}$. 
%If the point $y$ is totally ordered, the corresponding subgradient of the \lovasz{} extention is unique.  
Seen in this way,  the \lovasz{} extension then takes an extremely simple form: $\hat{f}(w) = \langle h^f_{\sigma_w}, w \rangle$.

We now point out an interesting property related to the extreme subgradients of $\hat{f}$ . Define $\mathcal P(\sigma)$ as a $n-$simplex corresponding to a permutation $\sigma$ (or chain $\mathcal C^{\sigma}: \emptyset \subset S^{\sigma}_1 \subset \cdots \subset S^{\sigma}_n = V$). In other words, $\mathcal P(\sigma) = \text{conv}(1_{S^{\sigma}_i}, i = 1,2, \cdots, n$). It's easy to see that $\mathcal P(\sigma) \subseteq [0, 1]^n$. \arxiv{In particular it carves out a particular polytope within the unit hypercube. Then we have the following result.}
\begin{lemma} (Lemma 6.19, ~\cite{fujishige2005submodular}) \label{lem2.3}
Given a permutation $\sigma \in \Sigmas$, for every vector $y \in \mathcal P(\sigma)$ the vector $h^f_{\sigma}$ is an extreme subgradient of $\hat{f}$ at $y$. If $y$ belongs to the (strict) interior of $\mathcal P(\sigma)$, $h^f_{\sigma}$ is a unique subgradient corresponding to $\hat{f}$ at $y$.
\end{lemma}
The above lemma points out a critical fact about the subgradients of the \lovasz{} extension, in that they depend only on the total ordering of a vector and are independent of the vector itself. This also implies that if $y$ is totally ordered (it belongs to the interior of $\mathcal P(\sigma_y)$) then $\partial_{\hat{f}}(y)$ consists of a single (unique) subgradient. %Hence even if two vectors are completely different, if they have the same partial or total ordering, the corresponding extreme subgradients will be identical. 
Hence, two entirely different but identically ordered vectors will have identical extreme subgradients.
This fact is important when defining and understanding the properties of the LB divergence. \arxiv{Also, $\mathcal P(\sigma)$ carves a unique polytope in $[0, 1]^n$ and the following lemma provides some insight into this:}

\arxiv{\begin{lemma} (Lemma 6.18, ~\cite{fujishige2005submodular})
Two polytopes $\mathcal P(\sigma_1)$ and $\mathcal P(\sigma_2)$ share a face if and only if permutations $\sigma_1$ and $\sigma_2$ are such that $\exists k \leq n$, with $S^{\sigma_1}_j = S^{\sigma_2}_j, \forall j \not= k$.
\end{lemma}
That is, two different permutation polytopes are adjacent to each other only if
the corresponding permutations are related by a transposition of
adjacent elements. E.g., if $n = 5$ the permutations $\{1,2,3,5,4\}$
and $\{1,2,3,4,5\}$ share a face, while the permutations
$\{1,2,5,4,3\}$ and $\{1,2,3,4,5\}$ do not. These properties of the
\lovasz{} extension, as we shall see, play a key role in the
properties and applications of the LB divergence.}

\subsection{The \lovasz{} Bregman divergences}
%Having now investigated the \lovasz{} extention the generalized Bregman divergence, 
We are now in position to define the \lovasz-Bregman divergence. Throughout this paper, we restrict $\dom(\hat{f})$ to be $[0, 1]^n$. For the applications we consider in this paper, we lose no generality with this assumption, since the scores can easily be scaled to lie within this volume.\looseness-1

Consider the case when $y$ is totally ordered, and correspondingly $|\Sigmas_y| = 1$. It follows then from Lemma~\ref{lem2.3} that there exists a unique subgradient and $\h{\hat{f}}(y) = h_{\sigma_y}^f$. Hence for any $x \in [0, 1]^n$, we have from Eqn.~\eqref{eq:genBregman} that~\cite{rkiyersubmodBregman2012}:
%We now define the \lovasz{} Bregman divergence. Notice that since the \lovasz{} extention is non-differentiable, the \lovasz{} Bregman divergence needs to be defined via the generalized Bregman divergence~\cite{telgarsky2012agglomerative,kiwiel}. For simplicity, we currently assume (though we shall relax this later), that we consider only vectors which are totally ordered. We do not loose much with this assumption, since the points which are not totally ordered are of measure zero. In this case, there exists a unique subgradient and the \lovasz{} Bregman divergence is defined as~\cite{rkiyersubmodBregman2012}:
\begin{align}
\LB(x, y) = \hat{f}(x) - \langle x, h^f_{\sigma_y} \rangle = \langle x, h^f_{\sigma_x} - h^f_{\sigma_y} \rangle
\end{align}
\begin{table*}
\centerline{
\begin{tabular}{|c|c|c|c|}
	\hline
& $f(X)$ & $\hat{f}(x)$ & $\LB(x, y)$  \\\hline
1) & $|X||V \backslash X|$ & $\sum_{i < j} |x_i - x_j|$ & $\sum_{i < j} |x_{\sigma(i)} - x_{\sigma(j)}| I(\sigma_x^{-1} \sigma(i) > \sigma_x^{-1} \sigma(j))$\\
2) & $g(|X|)$ & $\sum_{i = 1}^k x(\sigma_x(i)) \delta_g(i)$ & $\sum_{i = 1}^n x(\sigma_x(i)) \delta_g(i) - \sum_{i = 1}^k x(\sigma_y(i)) \delta_g(i)$\\
3) & $\min\{|X|, k\}$ & $\sum_{i = 1}^k x(\sigma_x(i))$ & $\sum_{i = 1}^k x(\sigma_x(i)) - \sum_{i = 1}^k  x(\sigma(i))$ \\
4) & $\min\{|X|, 1\}$ & $\max_i x_i$ & $\max_i x_i - x(\sigma(1))$\\
5) & $\sum_{i = 1}^n |I(i \in X) - I(i+1 \in X)$ & $\sum_{i = 1}^n |x_i - x_{i+1}|$ & $\sum_{i = 1}^n |x_i - x_{i+1}|I(\sigma_x^{-1} \sigma(i) > \sigma_x^{-1} \sigma(i+1))$ \\
6) & $I(1 \leq |A| \leq n-1)$ & $\max_i x(i) - \min_i x(i)$ & $\max_i x(i) - x(\sigma(1)) - \min_i x(i) + x(\sigma(n))$ \\
7) & $I(A \neq \emptyset, A \neq V)$ & $\max_{i, j} |x_i - x_j|$ & $\max_{i, j} |x_i - x_j| - |x(\sigma(1) - x(\sigma(n))|$ \\
\hline
\end{tabular}
}
\caption{Examples of the LB divergences.  $I(.)$ is the Indicator fn.}
\label{tab:1}
\end{table*}
Notice that this divergence depends only on $x$ and $\sigma_y$, and is independent of $y$ itself. In particular, the LB divergence between a vector $x$ and any vector $y \in \mathcal P(\sigma)$ is the same for all  $y \in \mathcal P(\sigma)$  (Lemma~\ref{lem2.3}). We also invoke the following lemma from~\cite{rkiyersubmodBregmanextended2012}:
\begin{lemma} (Theorem 2.2, \cite{rkiyersubmodBregmanextended2012}) \label{iffneg}
Given a submodular function whose
  polyhedron contains all possible extreme points %(e.g., $f(X) =
%  \sqrt{|X|}$%
%  or more generally $f(v | V \setminus \{ v \} ) > 0, \forall v$
%)
 and $x$ which is totally ordered, $\LB(x, y) = 0$ if and only if $\sigma_x = \sigma_y$.
\end{lemma}

At first sight it seems that the class of submodular functions satisfying Lemma~\ref{iffneg} is very specific. We point out however that this class is quite general and many instances we consider in this paper belong to this class of functions. For example, it is easy to see that the class of submodular functions $f(X) = g(|X|)$ where $g$ is a concave function satisfying $g(i) - g(i-1) \neq g(j) - g(j-1)$ for $i \neq j$ belong to this class.

Hence the \lovasz-Bregman divergence is score based permutation divergence, and we denote it as:
\begin{align}\label{LBperm}
\LB(x || \sigma) = \langle x, h^f_{\sigma_x} - h^f_{\sigma} \rangle
\end{align}
%It is clear from the above expression that this is a score based permutation divergence. 
\arxiv{Observe that this form also bears resemblance with Eqn.~\eqref{gordoneq}. In particular, Eqn.~\eqref{gordoneq} defines a distance between a vector $x$ and  a subgradient. On the other hand, Eqn.~\eqref{LBperm} is defined via a permutations. Since the extreme subgradients of the \lovasz{} extension have a one to one mapping with the permutations, Eqn.~\eqref{LBperm} can be seen as an instance of Eqn.~\eqref{gordoneq}.} As we shall see in the next section, this divergence has a number of properties akin to the standard permutation based distance metrics. Since a large class of submodular functions satisfy the above property (of having all possible extreme points), the \lovasz-Bregman divergence forms a large class of divergences. \looseness-1

\arxivalt{Note that $y$ may not always be totally ordered. In this case, we resort to the notion of subgradient map defined in section~\ref{GenBregSec}. A natural choice of a subgradient map in such cases is:
\begin{align} \label{extsubmap}
\h{\hat{f}}(y) = \frac{\sum_{\sigma \in \Sigmas_y} h_{\sigma}^f}{|\Sigmas_y|}
\end{align}
Assuming $f$ is a monotone non-decreasing non-negative and normalized submodular function (we shall see that we can assume this without loss of generality), then it's easy to see that $0 \in \partial_{\hat{f}}(0)$, and correspondingly we assume that $\h{\hat{f}}(0) = 0$. Notice that here $\LB(x, 0) = \hat{f}(x)$.}{The case when $y$ is not totally ordered can be handled similarly~\cite{rkiyerlovBregmanext2013}.}

\arxiv{Another related divergence is a generalized Bregman divergence obtained through $\phi(x) = \hat{f}(|x|)$. We denote this divergence as $d_{|\hat{f}|}(x, y)$, and it is easy to see that $d_{|\hat{f}|}(x, y) = \LB(x, y), \forall x, y \in \mathbb{R}^n_+$. In other words, this generalizes the \lovasz-Bregman divergence to other orthants, and hence we call this the extended \lovasz-Bregman divergence. 
% Similarly, when $f(X) = |X|$, we get back the divergences in example~\ref{ex1}. 
In a manner similar to the \lovasz-Bregman divergence, we can also define a valid subgradient map. The only difference here is that the there are additional point of non-differentiability at the change of orthants. This divergence captures  a distance between orderings while simultaneously considering the signs of the vectors. For example two vectors $[-1, -2]$ and $[2, 1]$ though having the same ordering are different from the perspective of the extended \lovasz-Bregman divergence. The LB divergence on the other hand just considers the ordering of the elements in a vector.\looseness-1}

\subsection{\lovasz{} Bregman Divergence Examples}
Below is a partial list of some instances of the \lovasz-Bregman divergence. We shall see that a number of these are closely related to many standard permutation based metrics. Table~\ref{tab:1} considers several other examples of LB divergences.

\paragraph{Cut function and symmetric submodular functions: }
A fundamental submodular function, which is also symmetric, is the
graph cut function. This is $f(X) = \sum_{i \in X} \sum_{j \in V
  \backslash X} d_{ij}$. The \lovasz{} extension of $f$ is $\hat{f}(x)
= \sum_{i, j} d_{ij} (x_i - x_j)_+$~\cite{frbach1}. The LB divergence
corresponding to $\hat{f}$ then has a nice form: \arxivalt{
 \begin{align} \LB(x || \sigma) &= \sum_{i, j} d_{ij} |x_i -
      x_j| I(x_i < x_j, \sigma^{-1}(i) < \sigma^{-1}(j))
      \nonumber\\
      &= \sum_{i, j} d_{ij} |x_i - x_j| I(\sigma_x^{-1}(i) > \sigma_x^{-1}(j), \sigma^{-1}(i) < \sigma^{-1}(j)) \nonumber\\
      &= \sum_{i, j: \sigma^{-1}(i) < \sigma^{-1}(j)} d_{ij} |x_i - x_j| I(\sigma_x^{-1}(i) > \sigma_x^{-1}(j)) \nonumber \\
\label{cut1}	&= \sum_{i < j} d_{\sigma(i)\sigma(j)} |x_{\sigma(i)} - x_{\sigma(j)}| I(\sigma_x^{-1} \sigma(i) > \sigma_x^{-1} \sigma(j)). 
\end{align}}{\small{
\begin{align}
\label{cut1} \LB(x || \sigma) = \sum_{i < j} d_{\sigma(i)\sigma(j)} |x_{\sigma(i)} - x_{\sigma(j)}| I(\sigma_x^{-1} \sigma(i) > \sigma_x^{-1} \sigma(j))
\end{align}} }\normalsize We in addition assume that $d$ is symmetric
(i.e., $d_{ij} = d_{ji}, \forall i, j \in V)$ and hence $f$ is also
symmetric. Indeed a weighted version of Kendall $\tau$ can be written
as $d^w_T(\sigma, \pi) = \sum_{i, j: i < j} w_{ij} I(\sigma^{-1}
\pi(i) > \sigma^{-1} \pi(j))$ and $\LB(x || \sigma)$ is exactly then a
form of $d^w_T(\sigma_x, \sigma)$, with $w_{ij} =
d_{\sigma(i)\sigma(j)}|x_{\sigma(i)} - x_{\sigma(j)}|$. Moreover, if
$d_{ij} = \frac{1}{|x_i - x_j|}$, we have $\LB(x || \sigma) =
d_T(\sigma_x, \sigma)$. \arxiv{Similarly, given the Kendall $\tau$,
  $d_T(\sigma, \pi)$, we can choose an arbitrary $x = x_{\sigma}$ such
  that $\sigma_x = \sigma$ and with an appropriate choice of $d_{ij}$,
  we have that $d_T(\sigma, \pi) = \LB(x_{\sigma} || \pi)$. }Hence, we
recover the Kendall $\tau$ for that particular $x$.
%This shows the dependence of $x$ in the divergence, since if $|x_i - x_j|$ is small $\forall i, j$, it implies that the ranking of $x$ is of low confidence. Hence the corresponding divergence will be small for every permutation, which is what is expected.

An interesting special case of this is when $f(X) = |X||V \backslash X|$, in which case we get:
\begin{align}
\LB(x || \sigma) = \sum_{i < j} |x_{\sigma(i)} - x_{\sigma(j)}| I(\sigma_x^{-1} \sigma(i) > \sigma_x^{-1} \sigma(j)). \nonumber 
\end{align}
\arxiv{This is also closely related to the  Kendall $\tau$ metric.}  

\paragraph{Cardinality based monotone submodular functions: }
Another class of submodular functions is $f(X) = g(|X|)$ for some concave function $g$. This form induces an interesting class of \lovasz{} Bregman divergences. In this case $h^f_{\sigma_x}(\sigma_x(i)) = g(i) - g(i-1)$. Define $\delta_g(i) = g(i) - g(i-1)$, then:
%Then it's easy to see that $\LB(x || \sigma) = \langle x, \delta_g \sigma_x^{-1} - \delta_g\sigma^{-1} \rangle$ and $\delta_g \sigma^{-1}(\sigma(i)) = \delta_g (i)$. Hence 
\begin{align} \label{cardconmod}
\LB(x || \sigma) = \sum_{i = 1}^n x[\sigma_x(i)] \delta_g (i) - \sum_{i = 1}^n x[\sigma(i)] \delta_g(i). 
\end{align}
Notice that we can start with any $\delta_g$ such that $\delta_g(1) \geq \delta_g(2) \geq \cdots \geq \delta_g(n)$, and through this we can obtain the corresponding function $g$. 
%However apart from that, as we shall see this divergence has a very intuitive form, and has a lot of properties desirable of permutation based divergences. 
Consider a specific example, with $\delta_g(i) = n - i$. Then, $\LB(x || \sigma) = \sum_{i = 1}^n x[\sigma(i)] i - x[\sigma_x(i)] i =  \langle x, \sigma^{-1} - \sigma_x^{-1} \rangle$. This expression looks similar to the Spearman's rule (Eqn.~\eqref{spearman}), except for being additionally weighted by $x$.

We can also extend this in several ways. For example, consider a
restriction to the top $m$ elements ($m < n$). Define $f(X) = \min\{g(|X|),
g(m)\}$. Then it is not hard to verify that:
\begin{align}\label{topm}
\LB(x || \sigma) = \sum_{i = 1}^m x[\sigma_x(i)] \delta_g(i) - \sum_{i = 1}^m x[\sigma(i)] \delta_g(i). 
\end{align}
A specific example is $f(X) = \min\{|X|, m\}$, where
\begin{align}\label{topm2}
\LB(x || \sigma) = \sum_{i = 1}^m x(\sigma_x(i)) - x(\sigma(i)). 
\end{align}
In this case, the divergence between $x$ and $\sigma$ is the difference between the largest $m$ values of $x$ and the $m$ first values of $x$ under the ordering $\sigma$. Here the ordering is not really important, but it is just the sum of the top $m$ values and hence if $\sigma_x$ and $\sigma$, under $x$, have the same sum of first $m$ values, the divergence is zero (irrespective of their ordering or individual element valuations). We can also define $\delta_g$, such that $\delta_g(1) = 1$ and $\delta_g(i) = 0, \forall i \not= 1$. Then, $\LB(x || \sigma) = \max_j x(j) - x(\sigma(1))$ (this is equivalent to Eqn.~\eqref{topm2} when $m = 1$). In this case, the divergence depends only on the top value, and if $\sigma_x$ and $\sigma$ have the same leading element, the divergence is zero. \looseness-1

\subsection{\lovasz{} Bregman as ranking measures}
In this section, we show how the \lovasz{} Bregman subsumes and is closely related to several commonly used loss functions in Information Retrieval connected to ranking.

\paragraph{The Normalized Discounted Cumulative Gain (NDCG): } The NDCG metric~\cite{jarvelin2000ir} is one of the most widely used ranking measures in web search. Given a relevance vector $r$, where the entry $r_i$ typically provides the relevance of a document $i \in \{1, 2, \cdots, n\}$ to a query, and an ordering of documents $\sigma$, the NDCG loss function with respect to a discount function $D$ is defined as:
\begin{align}\label{ndcgeq}
\mathcal L(\sigma) = \frac{\sum_{i = 1}^k r(\sigma_r(i)) D(i) - \sum_{i = 1}^k r(\sigma(i)) D(i)}{\sum_{i = 1}^k r(\sigma_r(i)) D(i)} 
\end{align}
Here $k \leq n$ is often used as a cutoff. Intuitively the NDCG loss compares an ordering $\sigma$ to the best possible ordering $\sigma_r$.  The typical choice of $D(i) = \frac{1}{\log (1 + i)}$, though in general any decreasing function can be used. This function is closely related to a form of the LB divergence. In particular, notice that $\mathcal L(\sigma) \propto \sum_{i = 1}^k r(\sigma_r(i)) D(i) - \sum_{i = 1}^k r(\sigma(i)) D(i)$ (since the denominator of Eqn.~\eqref{ndcgeq} is a constant) which is form of Eqn.~\eqref{topm} with $m = k$ and choosing the function $g(i) = \sum_{j = 1}^i D(i)$. \arxiv{It is not hard to see that when $D$ is decreasing, $g$ is concave.}

\paragraph{Area Under the Curve: } Another commonly used ranking
measure is the Area under the curve~\cite{spackman1989signal}.  Unlike
NDCG however, this just relies on a partial ordering of the documents
and not a complete ordering. In particular denote $G$ as a set of
``good'' documents and $B$ as a set of ``bad'' documents. Then the
loss function $\mathcal L(\sigma)$ corresponding to an ordering of
documents $\sigma$ is
\begin{align}
\mathcal L(\sigma) = \frac{1}{|G||B|} \sum_{g \in G, b \in B} I(\sigma(g) > \sigma(b)).
\end{align}
This can be seen as an instance of LB divergence corresponding to the cut function by choosing $d_{ij} = \frac{1}{|G||B|}, \forall i, j$, $x_g = 1, \forall g \in G$ and $x_b = 0, \forall b \in B$. 

\section{\lovasz{} Bregman Properties}
In this section, we shall analyze some interesting properties of the
LB divergences. While many of these properties show
strong similarities with permutation based metrics, the \lovasz{}
Bregman divergence enjoys some unique properties, thereby providing
novel insight into the problem of combining and clustering ordered
vectors.

\paragraph{Non-negativity and convexity:}
The LB divergence is a divergence, in that $\forall x, \sigma, \LB(x || \sigma) \geq 0$. Additionally if the submodular polyhedron of $f$ has all possible extreme points, $\LB(x || \sigma) = 0$ iff $\sigma_x = \sigma$. Also the \lovasz-Bregman divergence $\LB(x || \sigma)$ is convex in $x$ for a given $\sigma$. 

%Hence any problem involving optimizing over $x$ can be effeciently performed. For example, let $\mathbb{S}_{\pi}$ denote $x \in \mathbb{R}^n: \sigma_x = \pi$. Then consider the problem $\min_{x \in \mathbb{S}_{\pi}} \LB(x || \sigma)$. In other words, given two permutations $\pi, \sigma$ the problem is to find a $x \in \mathbb{S}_{\pi}$ which has minimum \lovasz{} Bregman divergence to $\sigma$. This problem is a convex optimization problem, and hence can effeciently be performed.

\paragraph{Equivalence Classes:}
The LB divergence of submodular
  functions which differ only in a modular term are equal. Hence for a submodular function $f$ and a modular function $m$, $d_{\widehat{f + m}}(x || \sigma) = d_{\hat{f}}(x || \sigma)$. Since any submodular function can be expressed as a difference between a polymatroid and a modular function~\cite{cun82}, it follows that it suffices to consider polymatroid \arxiv{(or non-negative monotone submodular)} functions while defining the LB divergences.\looseness-1
  
\paragraph{Linearity and Linear Separation:}
The LB divergence is a linear operator in the submodular function $f$. Hence for two submodular functions $f_1, f_2, d_{\widehat{f_1 + f_2}}(x || \sigma) = d_{\hat{f_1}}(x || \sigma) + d_{\hat{f_2}}(x || \sigma)$. 
The LB divergence has the property
  of linear separation --- the set of points $x$ equidistant to two permutations $\sigma_1$ and $\sigma_2$ (i.e., $\{ x : d_{\hat{f}}(x || \sigma_1) =
  d_{\hat{f}}(x || \sigma_2)\}$) comprise a hyperplane. Similarly, for any $x$, the set of points $y$ such that $\LB(x, y) = $ constant, is $\mathcal P(\sigma_y)$.\looseness-1
  
\paragraph{Invariance over relabelings:}
The permutation based distance metrics have the property of being left invariant with respect to reorderings,
%. This property is fundamental for a permutation based distance metric, since arbitrary reassignings of the orderings should not change the notion of a distance. In the context of a permutation based metric, the property is, 
i.e., given permutations $\pi, \sigma, \tau$, $d(\pi, \sigma) = d(\tau\pi, \tau\sigma)$. 

While this property may not be true of the \lovasz{} Bregman
divergences in general, the following theorem shows that this is true
for a large class of them.
\begin{theorem}
  Given a submodular function $f$, such that $\forall \sigma, \tau \in
  \Sigmas$, $h^f_{\tau \sigma} = \tau h^f_{\sigma}$, $\LB(x || \sigma)
  = \LB(\tau x || \tau \sigma)$.
\end{theorem}
\arxiv{\begin{proof}
Recall that $\tau x = x \tau^{-1}$. Hence
$\LB(\tau x || \tau \sigma ) = \langle \tau x, h^f_{\sigma_x \tau} - h^f_{\sigma \tau} \rangle = \langle x \tau^{-1}, h^f_{\sigma_x} \tau^{-1} - h^f_{\sigma} \tau^{-1} \langle = \LB(x || \sigma)$
\end{proof}\looseness-1}
This property seems a little demanding for a submodular function. But a large class of submodular functions can be seen to have this property. In fact, it can be verified that any cardinality based submodular function has this property. 
\begin{corollary}
Given a submodular function $f$ such that $f(X) = g(|X|)$ for some function $g$, then $\LB(x || \sigma) = \LB(\tau x, \tau \sigma)$.
\end{corollary}
This follows directly from Eqn.~\eqref{cardconmod} and observing that
the extreme points of the corresponding polyhedron are reorderings of
each other. In other words, in these cases the submodular polyhedron
forms a permutahedron. This property is true even for sums of such
functions and therefore for many of the special cases which we have
considered. \arxiv{For example, notice that the cut function $f(X) =
  |X| |V \backslash X|$ and the functions $f(X) = g(|X|)$ are both
  cardinality based
  functions.}\looseness-1%consider the case when $f(X) = |X| |V \backslash X|$, in which case $\LB(x || \sigma) = \sum_{i, j, i < j} 2|x_i - x_j| I(\sigma_x^{-1}\sigma(i) < \sigma_x^{-1} \sigma(j))$. Since this is a cardinality based function, it follows that it satisfies this property.

%%The reason for this is that, we have $\LB(x || \sigma) = \langle x, k\sigma_x^{-1} - k \sigma^{-1} \langle$. Hence $\LB(x \tau || \sigma \tau) = \langle x \tau, k \tau^{-1} \sigma_x^{-1} - k \tau^{-1} \sigma^{-1} \rangle = \langle x, k\sigma_x^{-1} - k \sigma^{-1} \rangle = \LB(x || \sigma)$. 
\begin{figure*}
  \centering
  \subfloat[LB2D]{\label{fig:lb2d}\includegraphics[width=0.16\textwidth]{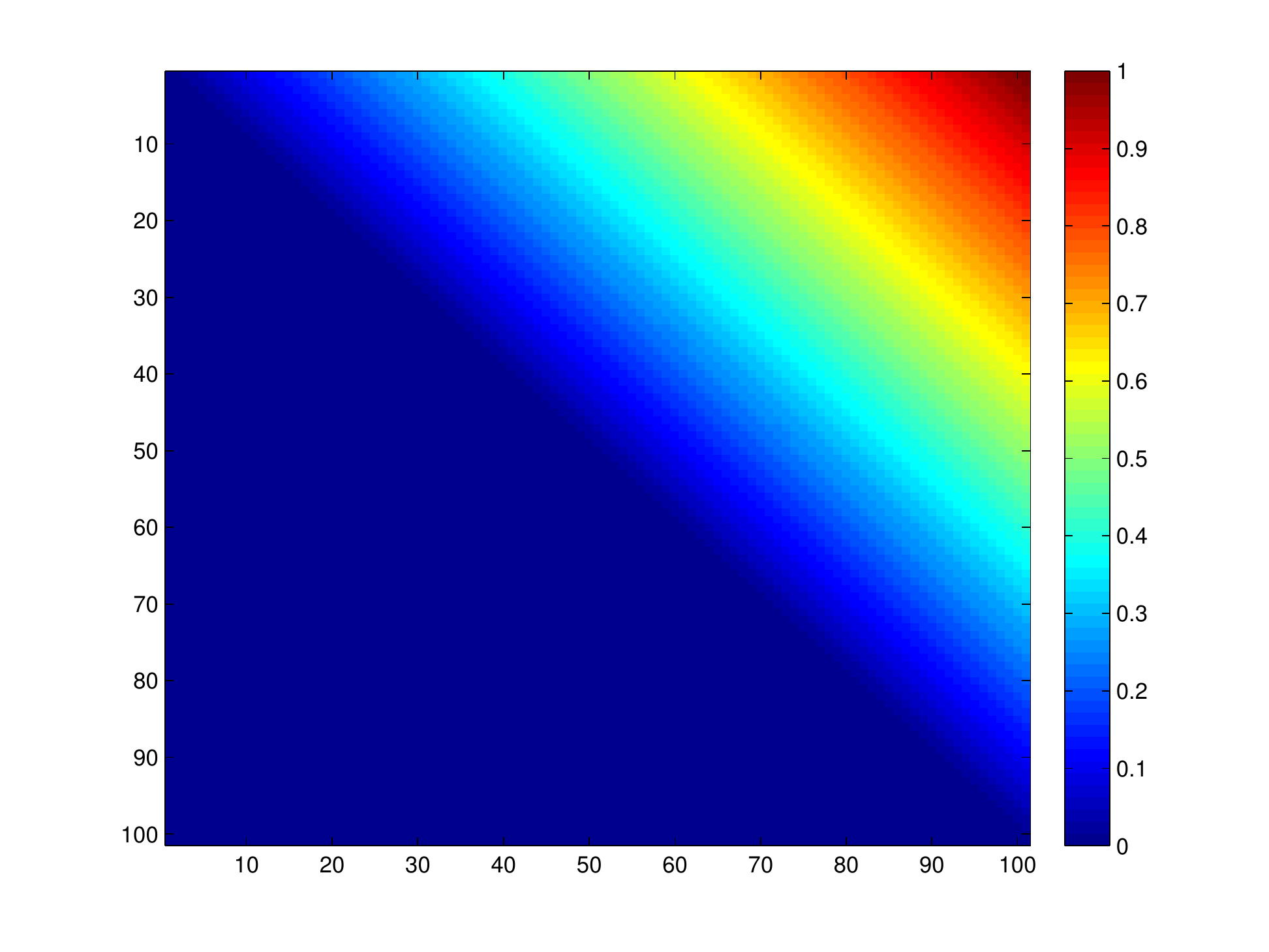}}
  ~
  \subfloat[LB3D1]{\label{fig:lb3d1}\includegraphics[width=0.16\textwidth]{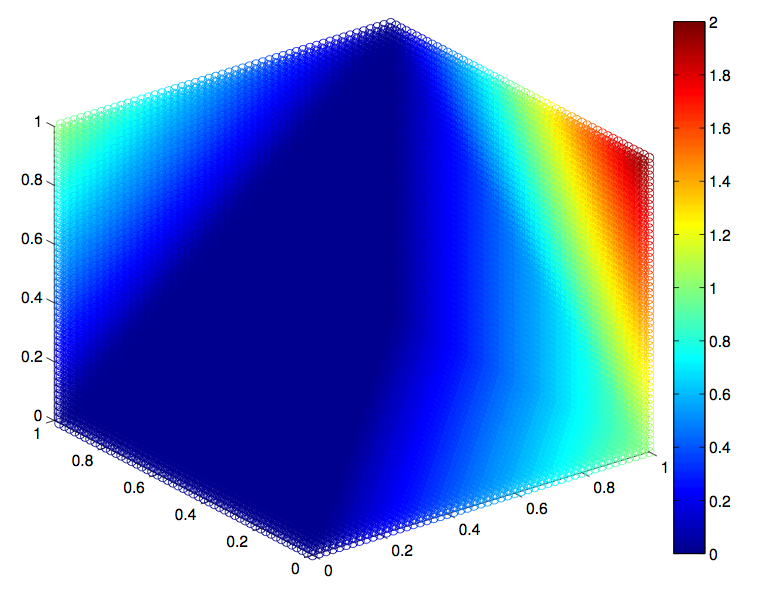}}
  ~
  \subfloat[LB3D2]{\label{fig:lb3d2}\includegraphics[width=0.16\textwidth]{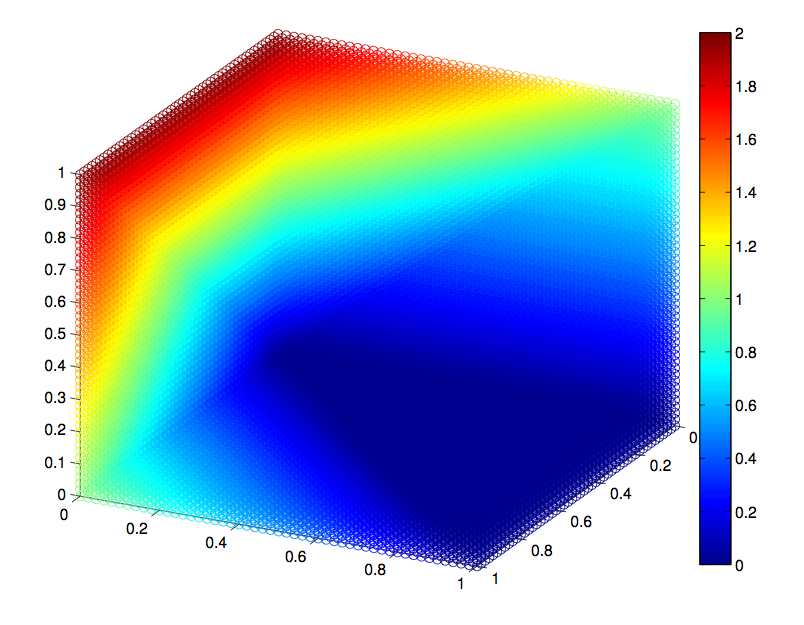}}
 \subfloat[KT2D]{\label{fig:lb2d}\includegraphics[width=0.16\textwidth]{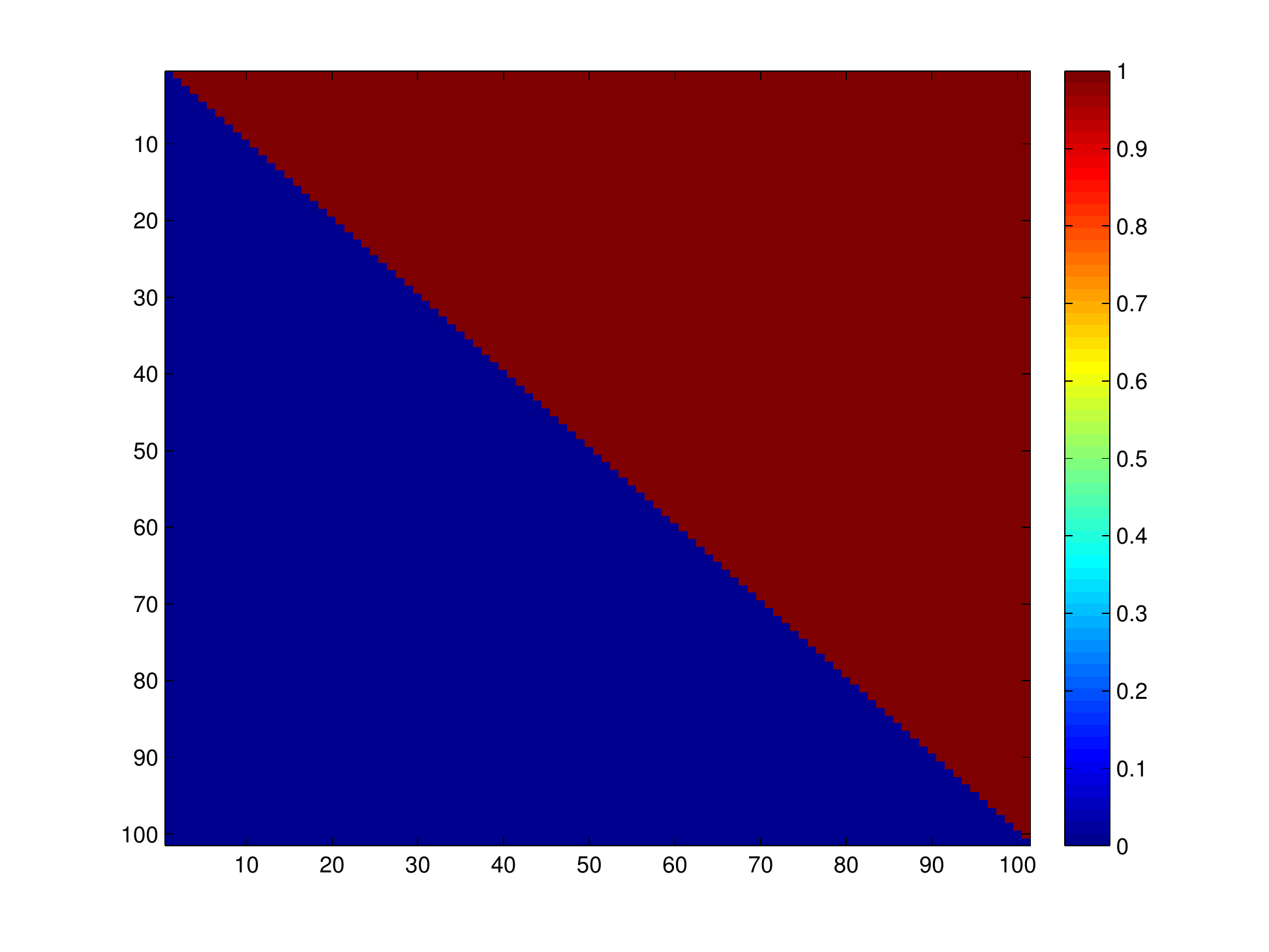}}
  \subfloat[KT3D1]{\label{fig:lb3d1}\includegraphics[width=0.16\textwidth]{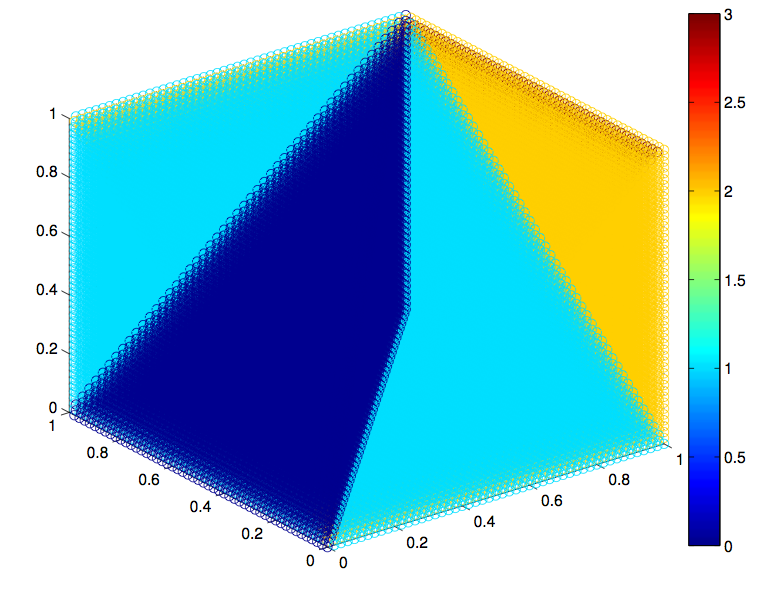}}
  \subfloat[KT3D2]{\label{fig:lb3d2}\includegraphics[width=0.16\textwidth]{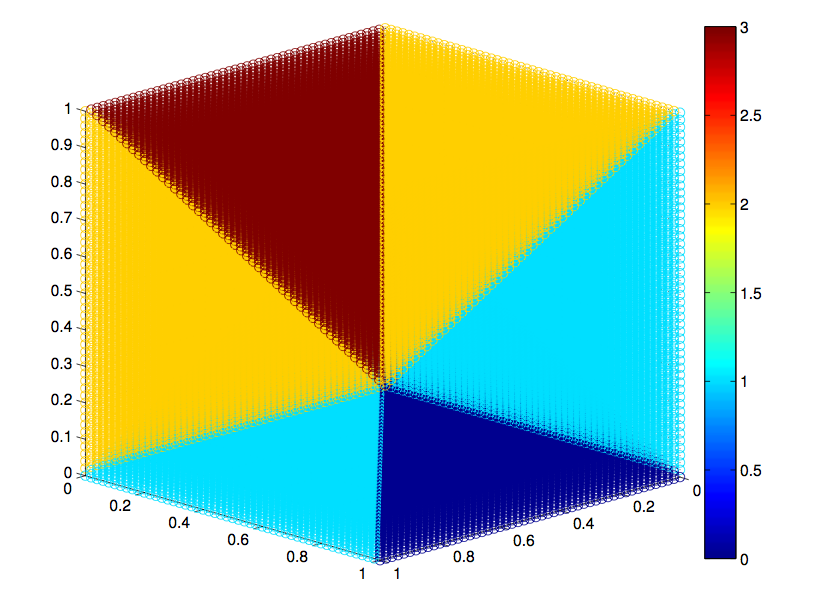}}
 %add desired spacing between images, e. g. ~, \quad, \qquad etc. (or a blank line to force the subfig onto a new line)
  \caption{A visualization of $\LB(x || \sigma)$ (left three) and $d_T(\sigma_x, \sigma)$ (right three). The figures shows a visualization in 2D, and two views in 3D for each, with $\sigma$ as $\{1,2\}$ and $\{1,2,3\}$ and $x \in [0, 1]^2$ and $[0, 1]^3$ respectively.}
  \label{fig:DivVis}
\end{figure*}

\paragraph{Dependence on the values and not just the orderings: }
We shall here analyze one key property of the LB divergence that is
not present in other permutation based divergences.
%Notice that the \lovasz{} Bregman divergence depends not only on the orderings $\sigma$ but also the vector $x$ itself. This property is what one would expect in a permutation based metric. For example
Consider the problem of combining rankings where, given a collection of scores $x^1, \cdots, x^n$, we want to come up with a joint ranking. An extreme case of this is where for some $x$ all the elements are the same. In this case $x$ expresses no preference in the joint ranking. Indeed it is easy to verify that for such an $x$, $\LB(x || \sigma) = 0, \forall \sigma$. Now given a $x$ where all the elements are almost equal (but not exactly equal), even though this vector is totally ordered, it expresses a very low confidence in it's ordering. We would expect for such an $x$, $\LB(x || \sigma)$ to be small for every $\sigma$. Indeed we have the following result:
\begin{theorem}
Given a monotone submodular function $f$ and any permutation $\sigma$, 
\begin{align}
\LB(x || \sigma) \leq \epsilon n (\max_j f(j) - \min_j f(j | V \backslash j)) \leq \epsilon n \max_j f(j) \nonumber
\end{align}
where $\epsilon = \max_{i, j} |x_i - x_j|$ and $f(j | A) = f(A \cup j) - f(A)$.
\end{theorem}
\arxiv{\begin{proof}
Decompose $x = \min_j x_j \mathbf{1} + r$, where $r_i = x_i - \min_j x_j$. Notice that $|r_i| \leq \epsilon$. Moreover, $\sigma_x = \sigma_r$ and hence $\LB(x || \sigma) = \LB(\min_j x_j \mathbf{1} || \sigma) + \LB(r || \sigma) = \LB(r || \sigma)$ since $\langle \mathbf{1}, h^f_{\sigma_r} - h^f_{\sigma} \rangle= f(V) - f(V) = 0$. Now, $\LB(r || \sigma) = \langle r,  h^f_{\sigma_r} - h^f_{\sigma} \rangle \leq ||r||_2 ||h^f_{\sigma_r} - h^f_{\sigma}||_2$. Finally note that $||r||_2 \leq \epsilon \sqrt{n}$ and $||h^f_{\sigma_r} - h^f_{\sigma}||_2 \leq \sqrt{n}  (\max_j f(j) - \min_j f(j | V \backslash j))$ and combining the two, we get the result. The second part follows from the monotonicity of $f$.
\end{proof}}
The above theorem implies that if the vector $x$ is such that all it's elements are almost equal, then $\epsilon$ is small and the LB divergence is also proportionately small. This bound can be improved in certain cases. For example for the cut function, with $f(X) = |X| |V \backslash X|$, we have that $\LB(x || \sigma) \leq \epsilon d_T(\sigma_x, \sigma) \leq \epsilon n(n-1)/2$, where $d_T$ is the Kendall $\tau$. \looseness-1

\paragraph{Priority for higher rankings: }
We show yet another nice property of the LB divergence with respect to a natural priority in rankings. This property has to do intrinsically with the submodularity of the generator function. We have the following theorem, that demonstrates this:
\begin{lemma}
Given permutations $\sigma, \pi$, such that $\mathcal P(\sigma)$ and $\mathcal P(\pi)$ share a face (say $S_k^{\sigma} \not= S_k^{\pi}$) and $x \in \mathcal P(\pi)$), then $\LB(x || \sigma) = (x_k - x_{k+1}) (f(\sigma_x(k) | S_{k-1}^{\sigma}) - f(\sigma_x(k) | S_k^{\sigma}))$.
\end{lemma}
This result directly follows from the definitions. Now consider the class of submodular function $f$ such that $\forall j, k \notin X, j \not= k, f(j | S) - f(j | S \cup k)$ is monotone decreasing as a function of $S$. An example of such a submodular function is again $f(X) = g(|X|)$, for a concave function $g$. Then it is clear that from the above Lemma that $\LB(x || \sigma)$ will be larger for smaller $k$. In other words, if $\pi$ and $\sigma$ differ in the starting of the ranking, the divergence is more than if $\pi$ and $\sigma$ differ somewhere towards the end of the ranking. This kind of weighting is more prominent for the class of functions which depend on the cardinality, i.e., $f(X) = g(|X|)$. Recall that many of our special cases belong to this class. Then we have that $\LB(x || \sigma) = \sum_{i = 1}^n \{x(\sigma_x(i)) - x(\sigma(i))\} \delta_g(i)$. Now since $\delta_g(1) \geq \delta_g(2) \geq \cdots \geq \delta_g(n)$, it then follows that if $\sigma_x$ and $\sigma$ differ in the start of the ranking, they are penalized more.\looseness-1

\paragraph{Extensions to partial orderings and top $m$-Lists:}
So far we considered notions of distances between a score $x$ and a complete permutation $\sigma$. %The \lovasz{} Bregman divergences, however, can be easily extended to partial rankings. \arxiv{First notice that if $y$ is not totally ordered, we can define an appropriate subgradient map, as in equation~\eqref{extsubmap}. 
%This map chooses the minimum entropy subgradient. 
%Similar extensions have been considered in the case of the Kendall $\tau$ distance, by summing over all valid permutations which can be seen as extensions of the given partial ordering. In the context of the \lovasz-Bregman divergence, this has a natural interpretation of being the mean subgradient map.} In typical applications the score vectors $x$ may not given entirely
%and only partial information corresponding to $x$ is given. One way of implementing this in our framework is by assigning a score of zero to all those items which have not been assigned. 
Often we may not be interested in a distance to a total ordering $\sigma$, but just a distance to a say a top-$m$ list~\cite{klementiev2008unsupervised} or a partial ordering between elements~\cite{yue2007support, chakrabarti2008structured}.
The LB divergence also has a nice interpretation for both of these. In particular, in the context of top $m$ lists, we can use Eqn.~\eqref{topm}. This exactly corresponds to the divergence between different or possibly overlapping sets of $m$ objects. Moreover, if we are simply interested in the top $m$ elements without the orderings, we have Eqn.~\eqref{topm2}. A special case of this is when we may just be interested in the top value. Another interesting instance is of partial orderings, where we do not care about the total ordering. For example, in web ranking we often just care about the relevant and irrelevant documents and that the relevant ones should be placed above the irrelevant ones. We can then define a distance $\LB(x || \mathcal P)$ where $\mathcal P$ refers to a partial ordering by using the cut based \lovasz{} Bregman (Eqn.~\eqref{cut1}) and defining the graph to have edges corresponding to the partial ordering. For example if we are interested in a partial order $1 > 2, 3 > 2$ in the elements $\{1,2,3,4\}$, we can define $d_{1, 2} = d_{3,2} = 1$ with the rest $d_{ij} = 0$ in Eqn.~\eqref{cut1}. Defined in this way, the LB divergence then measures the distortion between a vector $x$ and the partial ordering $1 > 2, 3 > 2$. In all these cases, we see that the extensions to partial rankings are natural in our framework, without needing to significantly change the expressions to admit these generalizations. \looseness-1

\paragraph{\lovasz{}-Mallows model:}
In this section, we extend the notion of Mallows model to the LB divergence. We first define the Mallows model for the LB divergence:
\begin{align}\label{LBM}
p(x | \theta, \sigma) = \frac{\exp(-\theta \LB(x || \sigma))}{Z(\theta, \sigma)}, \;\; \theta \geq 0.
\end{align}
For this distribution to be a valid probability distribution, we
assume that the domain $\mathcal D$ of $x$ to be a bounded set (say
for example $[0, 1]^n$). We also assume that the domain is symmetric
over permutations (i.e., for all $\sigma \in \Sigmas$, if $x \in
\mathcal D, x\sigma \in \mathcal D$. Unlike the standard Mallow's
model, however, this is defined over scores (or valuations) as
opposed to permutations.

Given the class of LB divergences defining a probability distribution over such a symmetric set (i.e., the divergences are invariant over relabelings) it follows that $Z(\theta, \sigma) = Z(\theta)$. The reason for this is:
\begin{align*}
Z(\theta, \sigma) &= \int_x \exp(-\theta \LB(x, \sigma)) dx  \\
&= \int_x \exp(-\theta \LB(x\sigma^{-1}, \sigma_0)) dx\nonumber\\
&= \int_{x^{\prime}} \exp(-\theta \LB(x^{\prime}, \sigma_0)) dx^{\prime} 
= Z(\theta)
\end{align*}
where $\sigma_0 = \{1, 2, \cdots, n\}$. 
We can also define an extended Mallows model for combining rankings, analogous to~\cite{lebanon2002cranking}. Unlike the Mallows model however this is a model over permutations given a collection of vectors $\mathcal X = \{x_1, \cdots, x_n\}$ and parameters $\Theta = \{\theta_1, \cdots, \theta_n\}$. 
\begin{align}\label{ELBM}
p(\sigma | \Theta, \mathcal X) =\frac{\exp(-\sum_{i = 1}^n \theta_i \LB(x_i || \sigma))}{Z(\Theta, \mathcal X)}
\end{align}
This model can be used to combine rankings using the LB divergences,
in a manner akin to Cranking~\cite{lebanon2002cranking}. This extended
\lovasz-Mallows model also admits an interesting Bayesian
interpretation, thereby providing a generative view to this
model:\looseness-1
\begin{align}
p(\sigma | \Theta, \mathcal X) \propto p(\sigma) \prod_{i = 1}^n p(x_i | \sigma, \theta_i).
\end{align}
Again this directly follows from the fact that in this case, in the \lovasz-Mallows model, the normalizing constants (which are independent of $\sigma$) cancel out. We shall actually see some very interesting connections between this conditional model and web ranking.

\arxiv{
\begin{figure}
  \centering
  \subfloat[Norm-1]{\label{fig:elb1}\includegraphics[width=0.17\textwidth]{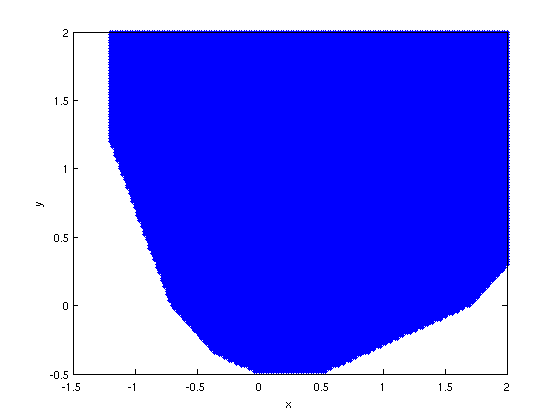}}
  ~
  \subfloat[Norm-2]{\label{fig:lb3d1}\includegraphics[width=0.17\textwidth]{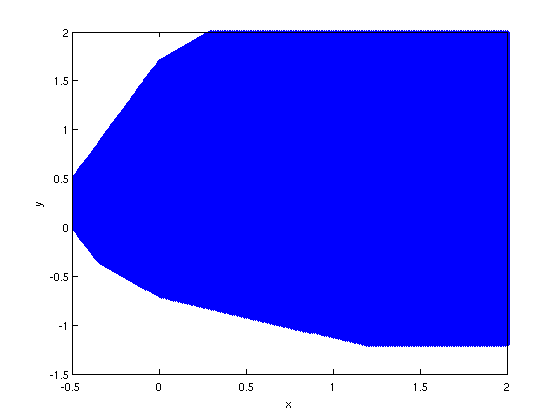}}
\caption{A visualization of the sub-level sets of $d_{|\hat{f}|}(x, y)$ for $y = [1, 2]$ (left) and $y = [2, 1]$ (right).}
\label{normvis}
\end{figure}
\paragraph{\lovasz{} Bregman as a regularizer:}
Consider the problem where we want to minimize a convex function subject to a permutation constraint. In these cases, the \lovasz{} Bregman can be seen as a regularizer since the problem $\min_{x: x \in \mathcal P(\sigma)} \psi(x)$ is equivalent to $\min_{x: \LB(x, \sigma) = 0} \psi(x)$. Alternatively, we may consider:\looseness-1
\begin{align}\label{eq:1}
\min_{x} \psi(x) + \lambda \LB(x || \sigma).
\end{align}
The above problem still remains a convex optimization problem since the LB divergence is convex in it's first argument. This problem is intimately related to the proximal methods investigated in~\cite{bach2010structured}. Indeed, Eqn.~\ref{eq:1} is equivalent to $\min_x \psi(x) - \lambda \langle h^{\sigma}_f, x \rangle + \hat{f}(x)$ and is of the form $\min_x \phi(x) + \hat{f}(x)$, and hence the methods in~\cite{bach2010structured} apply to this problem. Also, the connection between the LB divergences and the submodular norms can be made more explicit in the following manner. Consider the extended \lovasz-Bregman divergences $d_{|\hat{f}|}$. When $f$ is monotone, $0 \in \partial_{|\hat{f}|}(0)$ and hence choosing an appropriate subgradient map ensures that $d_{|\hat{f}|}(x, 0) = \hat{f}(|x|)$, which are the polyhedral norms defined in~\cite{bach2010structured}.

Moreover, we get other interesting norm structures when $y \neq 0$. For example, Figure~\ref{normvis} shows a visualization of the sub-level sets of $d_{|\hat{f}|}(x, y)$ for two choices of $y$ with different orderings. We see that these sublevel sets are open and naturally show preference to the orderings defined via the particular $y$.}
\arxiv{\begin{figure*}
  \centering
  \subfloat[LB2Dclus]{\label{fig:lb2dclus}\includegraphics[height=0.14\textwidth]{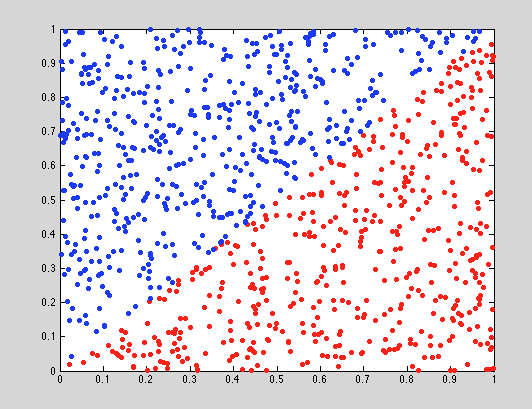}}
  ~
  \subfloat[LBEuc2Dclus]{\label{fig:lbeuc2dclus}\includegraphics[height=0.14\textwidth]{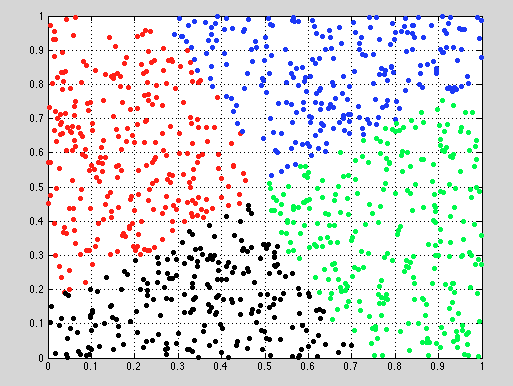}}
 ~
  \subfloat[LB3Dclus]{\label{fig:lb3dclus}\includegraphics[height=0.14\textwidth]{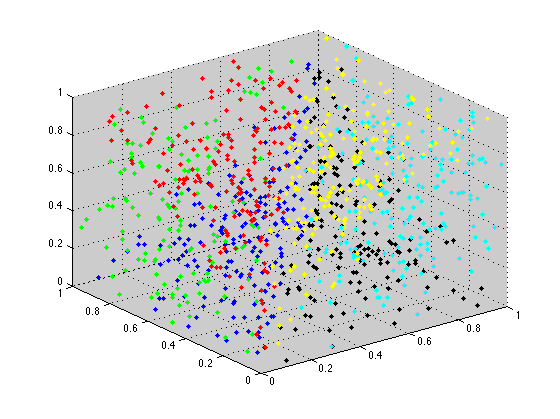}}
~
 \subfloat[LBmax3Dclus]{\label{fig:lb3dmaxclus}\includegraphics[height=0.14\textwidth]{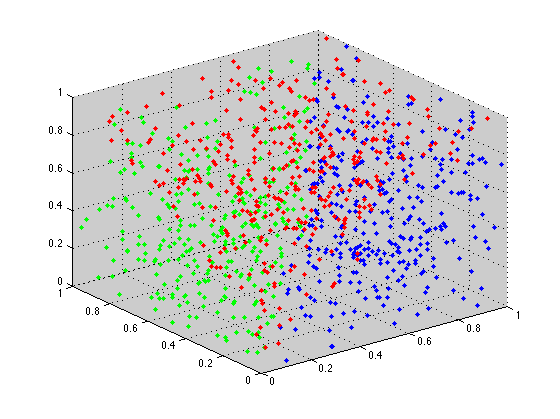}}
 %add desired spacing between images, e. g. ~, \quad, \qquad etc. (or a blank line to force the subfig onto a new line)
  \caption{Clustering with the (a) \lovasz-Bregman in 2D, (b) Combination of \lovasz-Bregman and Euclidean in 2D, (c) \lovasz-Bregman in 3D and (d) \lovasz-Bregman with top $1$-list in 3D. All these use $f(X) = \sqrt{|X|}$}
  \label{fig:ClusVis}
\end{figure*}}
\section{Applications}
\paragraph{Rank Aggregation:}
As argued above, the LB divergence is a natural model for the problem of combining scores, where both the ordering and the valuations are provided. If we ignore the values, but just consider the rankings, this then becomes rank aggregation. A natural choice in such problems is the Kendall $\tau$ distance~\cite{lebanon2002cranking, klementiev2008unsupervised, meilua2007consensus}. On the other hand, if we consider only the values without explicitly modeling the orderings, then this becomes an incarnation of boosting~\cite{freund2003efficient}. The \lovasz-Bregman divergence tries to combine both aspects of this
problem -- by combining orderings using a permutation based
divergence, while simultaneously using the additional information of
the confidence in the orderings provided by the valuations. We can then pose this problem as: 
 \looseness-1
\begin{align} \label{meaneq} 
  \sigma \in
  \argmin_{\sigma^{\prime} \in \Sigmas} \sum_{i = 1}^n d_{\hat{f}}(x^i ||
  \sigma^{\prime})
\end{align}
The above notion of the representative ordering (also known as the mean ordering) is very common in many applications~\cite{banerjee2005} and has also been used in the context of combining rankings~\cite{meilua2007consensus, lebanon2002cranking, klementiev2008unsupervised}. Unfortunately this problem in the context of the permutation based metrics were shown to be NP hard~\cite{bartholdi1989voting}.
Surprisingly for the LB divergence this problem is easy (and has a closed form). In particular, the representative permutation is exactly the ordering corresponding to the arithmetic mean of the elements in $\mathcal X$.\looseness-1
\begin{lemma}~\cite{rkiyersubmodBregman2012} \label{lemmamean}
Given a submodular function $f$, the \lovasz{} Bregman representative (Eqn.~\eqref{meaneq}) is $\sigma = \sigma_{\mu}$, where $\mu = \frac{1}{n} \sum_{i = 1}^n x^i$
\end{lemma}
This result builds on the known result for Bregman
divergences~\cite{banerjee2005}. This seems somewhat surprising at
first. Notice, however, that the arithmetic mean uses additional
information about the scores and it’s confidence, as opposed to just
the orderings. In this context, the result then seems reasonable since
we would expect that the representatives be closely related to the
ordering of the arithmetic mean of the objects. We shall also see that
this notion has in fact been ubiquitously but unintentionally used in
the web ranking and information retrieval communities.\looseness-1

We illustrate the utility of the \lovasz{} Bregman rank aggregation through the following argument.  Assume that a particular vector $x$ is uninformative about the true ordering (i.e, the values of $x$ are almost equal). Then with the LB divergence and any permutation $\pi$, $d(x || \pi) \approx 0$, and hence this vector will not contribute to the mean ordering. Instead if we use a permutation based metric, it will ignore the values but consider only the permutation. As a result, the mean ordering tends to consider such vectors $x$ which are uninformative about the true ordering. As an example, consider a set of scores: $\mathcal X = \{1.9, 2\}, \{1.8, 2\}, \{1.95, 2\}, \{2, 1\}, \{2.5, 1.2\}$. The representative of this collection as seen by a permutation based metric would be the permutation $\{1, 2\}$ though the former three vectors have very low confidence. The arithmetic mean of these vectors is however $\{2.03, 1.64\}$ and the \lovasz{} Bregman representative would be $\{2, 1\}$.\looseness-1

The arithmetic mean also provides a notion of confidence of the population. In particular, if the total variation~\cite{frbach1} of the arithmetic mean is small, it implies that the population is not confident about it’s ordering, while if the variation is high, it provides a certificate of a homogeneous population. Figure~\ref{fig:DivVis} provides a visualization the \lovasz-Bregman divergence using the cut function and the Kendall $\tau$ metric, visualized in 2 and 3 dimensions respectively. We see the similarity between the two divergences and at the same time, the dependence on the ``scores'' in the \lovasz-Bregman case. \looseness-1%The 2-D and 3-D divergences have been shown with respect to permutations $\sigma = \{1, 2\}$ and $\sigma = \{1, 2, 3\}$ respectively.

\paragraph{Learning to Rank: } We investigate a specific instance of
the rank aggregation problem with reference to the problem of ``learning
to rank.'' A large class of algorithms have been proposed for this
problem -- see~\cite{liu2009learning} for a survey on this. A specific
class of algorithms for this problem have focused on maximum margin
learning using ranking based loss functions (see~\cite{yue2007support,
  chakrabarti2008structured} and references therein). While we have
seen that the ranking based losses themselves are instances of the LB
divergence, the feature functions are also closely related.

In particular, given a query $q$, we denote a feature vector
corresponding to document $i \in \{1, 2, \cdots, n\}$ as $x_i \in
\mathbb{R}^d$, where each element of $x_i$ denotes a quality of 
document $i$ based on a particular indicator or feature. Denote
$\mathcal X = \{x_1, \cdots, x_n\}$. We assume we have $d$ feature
functions (one might be for example a match with the title, another
might be pagerank, etc). Denote $x_i^j$ as the score of the
$j^{\mbox{th}}$ feature corresponding to document $i$ and $x^j \in
\mathbb{R}^n$ as the score vector corresponding to feature $j$ over
all the documents. In other words, $x^j = (x^j_1, x^j_2, \cdots,
x^j_n)$. One possible choice of feature function is:
\begin{align} \label{featurefn}
\phi(\mathcal X, \sigma) = \sum_{j = 1}^d w_j \LB(x^j || \sigma)
\end{align}
for a weight vector $w \in \mathbb{R}^d$. Given a particular weight
vector $w$, the inference problem then is to find the permutation
$\sigma$ which minimizes $\phi(\mathcal X, \sigma)$. Thanks to
Lemma~\ref{lemmamean}, the permutation $\sigma$ is exactly the
ordering of the vector $\sum_{j = 1}^n w_j x^j$. It is not hard to see
that this exactly corresponds to ordering the scores $w^{\top} x_i$
for $i \in \{1, 2, \cdots, n\}$. Interestingly many of the feature
functions used in~\cite{yue2007support, chakrabarti2008structured} are
forms closely related Eqn.~\eqref{featurefn}. In fact the motivation
to define these feature functions is exactly that the inference
problem for a given set of weights $w$ be solved by simply ordering
the scores $w^{\top} x_i$ for every $i \in \{1, 2, \cdots,
n\}$~\cite{chakrabarti2008structured}. We see that through
Eqn.~\eqref{featurefn}, we have a large class of possible feature
functions for this problem.

We also point out a connection between the learning to rank problem
and the \lovasz-Mallows model. In particular, recent
work~\cite{dubey2009conditional} defined a conditional probability
model over permutations as:
\begin{align}\label{condmodel}
p(\sigma | w, \mathcal X) = \frac{\exp(w^{\top} \phi(\mathcal X, \sigma))}{Z}. 
\end{align}
This conditional model is then exactly the extended \lovasz-Mallows
model of Eqn.~\eqref{ELBM} when $\phi$ is defined as in
Eqn.~\eqref{featurefn}. The conditional models used
in~\cite{dubey2009conditional} are in fact closely related to this and
correspondingly Eqn.~\eqref{condmodel} offers a large class of
conditional ranking models for the learning to rank problem.

\paragraph{Clustering:}
A natural generalization of rank aggregation is the problem of clustering. In this context, we assume a heterogeneous model, where the data is represented as mixtures of ranking models, with each mixture representing a homogeneous population. It is natural to define a clustering objective in such scenarios. Assume a set of representatives $\Sigma = \{\sigma^1, \cdots, \sigma^k\}$ and a set of clusters $\mathcal C = \{\mathcal C^1, \mathcal C^2, \cdots, \mathcal C^k\}$. The clustering objective is then:
$\min_{\mathcal C, \Sigma} \sum_{j = 1}^k \sum_{i: x_i \in \mathcal C_j} \LB(x_i || \sigma_i)$.
As shown in~\cite{rkiyersubmodBregman2012}, a simple k-means style algorithm finds a local minima of the above objective. Moreover due to simplicity of obtaining the means in this case, this algorithm is extremely scalable and practical.

\arxiv{One additional property of the \lovasz-Bregman divergences is that they are easily amenable to other Bregman divergences as well. Since the sum of Bregman divergences is also a Bregman divergence, the k-means algorithm will go through almost identically if we add other Bregman divergences (for example, the Euclidean distance). We demonstrate the clustering patterns obtained through these divergences in Figure~\ref{fig:ClusVis} for randomly sampled points in $2$ and $3$ D. The second figure (from left) demonstrates the combined clustering using the \lovasz-Bregman and squared-Euclidean distance, which shows additional dependence on the values through the Euclidean distance. Similarly the last figure (from left), provides a clustering using the top $m$ objective, with $m = 1$ (Eqn.~\eqref{topm}). \looseness-1}

\section{Discussion}
To our knowledge, this work is the first introduces the notion of
``score based divergences'' in preference and ranking based
learning. Many of the results in this paper are due to some
interesting properties of the \lovasz{} extension and Bregman
divergences. This also provides interesting connections between web
ranking and the permutation based metrics. This idea is mildly related
to the work of~\cite{tehrani2012preference} where they use the Choquet
integral (of which the \lovasz{} extension is a special case) for
preference learning. Unlike our paper, however, they do not focus on
the divergences formed by the integral. Finally, it will be
interesting to use these ideas in real world applications involving
rank aggregation, clustering, and learning to rank.\looseness-1

{\bf Acknowledgments:} We thank Matthai Phillipose, Stefanie Jegelka
and the melodi-lab submodular group at UWEE for discussions and the anonymous
reviewers for very useful reviews. This material is based upon work
supported by the National Science Foundation under Grant
No. (IIS-1162606), and is also supported by a Google, a Microsoft, and
an Intel research award.  
\bibliographystyle{plain}
\bibliography{../Combined_Bib/submod}
\end{document}